\documentclass[final,5p,times,authoryear]{elsarticle}

\usepackage{amssymb}
\usepackage{amsmath}
\usepackage{amsthm}
\newtheorem{theorem}{Theorem}

\usepackage{amsfonts}
\usepackage{graphicx}
\usepackage{subcaption}
\usepackage{color}
\usepackage{bm}

\usepackage{mathptmx} 
\usepackage{times} 
\usepackage{appendix}

\usepackage{multirow} 
\usepackage{makecell} 

\usepackage[colorlinks=true,linkcolor=blue,citecolor=blue,urlcolor=blue]{hyperref}

\journal{Control Engineering Practice}

\begin{document}

\begin{frontmatter}



\title{Insect-Scale Tailless Robot with Flapping Wings: A Simple Structure and Drive for Yaw Control}


\author[label1]{Tomohiko Jimbo}
\author[label1]{Takashi Ozaki}
\author[label1]{Norikazu Ohta}
\author[label1]{Kanae Hamaguchi}

            
\address[label1]{Toyota Central R\&D Labs., Inc., 480-1192 Aichi, Japan.\\
E-mail: \texttt{t-jmb@mosk.tytlabs.co.jp}}

\begin{abstract}
Insect-scale micro-aerial vehicles, especially lightweight, flapping-wing robots, are becoming increasingly important for safe motion sensing in spatially constrained environments such as living spaces. However, yaw control using flapping wings is fundamentally more difficult than using rotating wings. In this study, an insect-scale, tailless robot with four paired tilted flapping wings (weighing 1.52 g) was fabricated to enable simultaneous control of four states, including yaw angle. The controllability Gramian was derived to quantify the controllability of the fabricated configuration and to evaluate the effects of the tilted-wing geometry on other control axes. This robot benefits from the simplicity of directly driven piezoelectric actuators without transmission, and lift control is achieved simply by changing the voltage amplitude. However, misalignment or modeling errors in lift force can cause offsets. Therefore, an adaptive controller was designed to compensate for such offsets. Numerical experiments confirm that the proposed controller outperforms a conventional linear quadratic integral controller under unknown offset conditions. Finally, in a tethered and controlled flight experiment, yaw drift was suppressed by combining the tilted-wing arrangement with the proposed controller.
\end{abstract}




\end{frontmatter}


\section{Introduction}

Unmanned aerial vehicles (UAVs) are increasingly being used to monitor and inspect large environments such as farms, roads, buildings, and bridges. They are also expected to be employed in spatially constrained environments such as living spaces. In confined spaces, UAVs must possess the ability to hover, exhibit high maneuverability, and be lightweight to ensure safety in the event of collisions. Therefore, flapping-wing robots, inspired by insects and birds, are more suitable for constrained spaces than the rotating-wing robots commonly used in UAVs. 

\textcolor{black}{
High maneuverability and efficiency of flapping wings have been reported in various studies on the aerodynamics, kinematics and control of flapping-wing robots \citep{1999Dickinson+, 2012Christopher+, 2018Helbling+, 2019Phan+Review}. 
While the incorporation of a tail enables passive flight and extends flight duration \citep{2016Rosen, 2016delfly}, it also diminishes maneuverability and complicates hover control due to the tail's damping effect and sensitivity to external disturbances such as wind. 
In this paper, we consider a tailless configuration to achieve full controllability in a compact design, focusing on how simple flapping mechanisms can maintain yaw stability.
}

\begin{figure}[t]
\centering
\includegraphics[width=8cm]{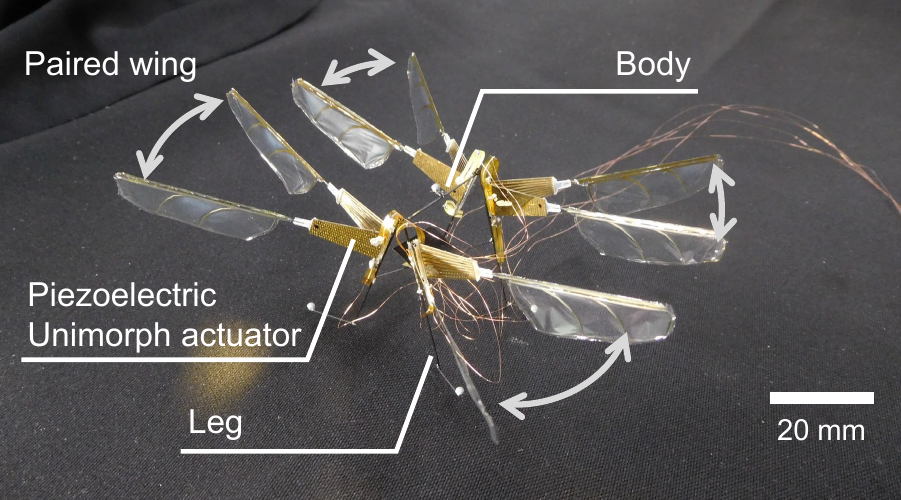}
\caption{
\textcolor{black}{
Flapping-wing robot fabricated with four paired wings tilted by $\beta=20$ deg and $\gamma=60$ deg,
equipped with eight piezoelectric actuators, a glass-fiber-reinforced plastic body, carbon-fiber-reinforced plastic legs, a wingspan of $110$ mm, and a weight of $1.52$ g.
}
\label{f:flappingwingRobots}
}
\end{figure}

\textcolor{black}{
Several approaches to flapping-wing actuation have been studied and summarized in \citep{2018Helbling+, 2019Phan+Review, 2024Chen+}. 
}
\textcolor{black}{
Electromagnetic motors are a promising approach, and many robots have been proposed, often employing complex gear or linkage mechanisms to achieve reciprocating motion \citep{2009Croon+, 2012Keennon+, 2015Nguyen+, 2018Karasek+, 2019Phan+}.
\textcolor{black}{
More recent work has explored actively controlled linkage mechanisms to improve roll and pitch control torques \citep{2025Wu+}.
}
They are heavy, weighing $10$ g or more. 
In contrast, 
Piezoelectric actuators, by enabling direct flapping motion without mechanical transmission, make it possible to design 
extremely insect-scale robots weighing less than $1$ g \citep{2008Wood, 2013Ma+, 2018_3rd_Ozaki, 2019Chen+}. 
}
\textcolor{black}{
These flapping-wing actuations enable lightweight and inherently safe designs, but there are still challenges in achieving sufficient yaw control capability due to the limited yaw torque and coupling among control axes.
}

Yaw control, which keeps onboard sensors aligned in the required direction, simplifies the design of the horizontal controller. However, generating yaw torque using flapping wings is complex \citep{2014Helbling+, 2015Teoh, 2019Fuller}. 
As demonstrated in \citep{2019Steinmeyer+}, yaw torque can be generated by adjusting the speed ratio between the upstroke and downstroke of piezoelectric actuators.
However, this process is sensitive to manufacturing errors and parameter variations \citep{2021Chukewad+}. 
Therefore, yaw control with flapping wings should be achieved through a simple structure and a simple control signal.

\textcolor{black}{
Recent work has demonstrated that high flight performance can be achieved in insect-scale flapping-wing robots by exploiting advanced computation and learning-based control frameworks. In particular, \cite{2025Hsiao+} reported impressive aggressive flight maneuvers using a learning-based model predictive control approach on a 750-mg flapping-wing robot. 
However, this performance is achieved by relying on substantial external computational resources, making onboard implementation nontrivial. Moreover, due to the platform’s design, active yaw torque generation is not considered. In practice, yaw motion is mitigated through manual pre-flight calibration of the wing-mount orientation, thereby ensuring that the yaw angular drift rate remains below 200 deg/s. Consequently, yaw-stabilized hovering flight is not explicitly addressed.
On the other hand, model-based optimal control provides a computationally efficient alternative for insect-scale flapping-wing robots. An LQR-based controller demonstrated in \cite{npjRobotics2025} achieved stable near-hover flight and trajectory tracking, while yaw control remains outside the scope of that study.
}

We previously investigated a direct-driven flapping wing using a piezoelectric unimorph actuator \citep{2018_1st_Ozaki, 2018_2nd_Ozaki}. It is simple because it has no displacement-enhancing structure (transmission-free structure), 
and the wings are driven by simply changing the voltage amplitude at the resonance frequency. 
Furthermore, to suppress the coupling effect between the wings and the body \citep{2018_3rd_Ozaki}, we fabricated a robot with three paired-wing actuators \citep{2020_Jimbo}.  
However, because it could only control three states simultaneously, it was unable to generate yaw torque.
Therefore, we propose an insect-scale flapping-wing robot weighing 1.52 g (Fig.~1), which generates yaw torque using a tilting arrangement of four paired-wing actuators. 
\textcolor{black}{
The effects of the tilted wings on controls other than yaw were examined using the controllability Gramian to evaluate the characteristics of the fabricated configuration.
}
Furthermore, despite the simple structure, manufacturing errors can occur, resulting in a lift-force offset.
In addition, it is difficult to accurately measure the instantaneous lift forces generated during the flapping motion. Therefore, we modeled the offset force and torque acting on the body and designed an adaptive control system. The major contributions of this study are summarized below:
\begin{itemize}
\item A novel flapping-wing robot with four paired tilted wings is fabricated to enable yaw control. The wings are driven by piezoelectric direct drive actuators; there is no transmission, and the lift force can be controlled by simply changing the voltage amplitude. 
\textcolor{black}{
\item The controllability of the fabricated wing arrangement is analyzed using a controllability Gramian. 
The advantages of the tilting angle for yaw control, as well as the limitations for other controls, are clarified. 
}
\item An adaptive controller is designed to account for lift force offset. Through numerical experiments, we compared its performance with a linear quadratic integral (LQI) controller and confirmed its adaptability to an unknown lift offset.
\item We confirm that the yaw drift of the insect-scale flapping-wing robot can be suppressed through a controlled flight experiment.
\end{itemize}
\textcolor{black}{
The remainder of this paper is organized as follows:
\textcolor{black}{
Section~2 describes the structural design and direct-drive scheme of the flapping-wing robot, including the paired-wing configuration and the piezoelectric actuation mechanism.
}
In \textcolor{black}{Section~3}, a control-oriented model of the flapping-wing robot with four paired-wing actuators is derived,
and the uncertainties of the lift force and torque are formulated.
The controllability Gramian is also derived to provide a basis for analyzing the control characteristics of the fabricated configuration.
\textcolor{black}{Section~4} presents the design of an adaptive controller that compensates for unknown lift-force offsets.
In \textcolor{black}{Section~5}, numerical simulations are conducted to evaluate the controllability characteristics
and to compare the performance of the proposed controller with that of a conventional LQI controller.
\textcolor{black}{Section~6} demonstrates a controlled flight experiment of the flapping-wing robot.
Finally, \textcolor{black}{Section~7} concludes the paper.
}

\section{\textcolor{black}{Structural design and drive of flapping-wing robot}}

\begin{figure}[b!]
\centering
\includegraphics[width=8cm]{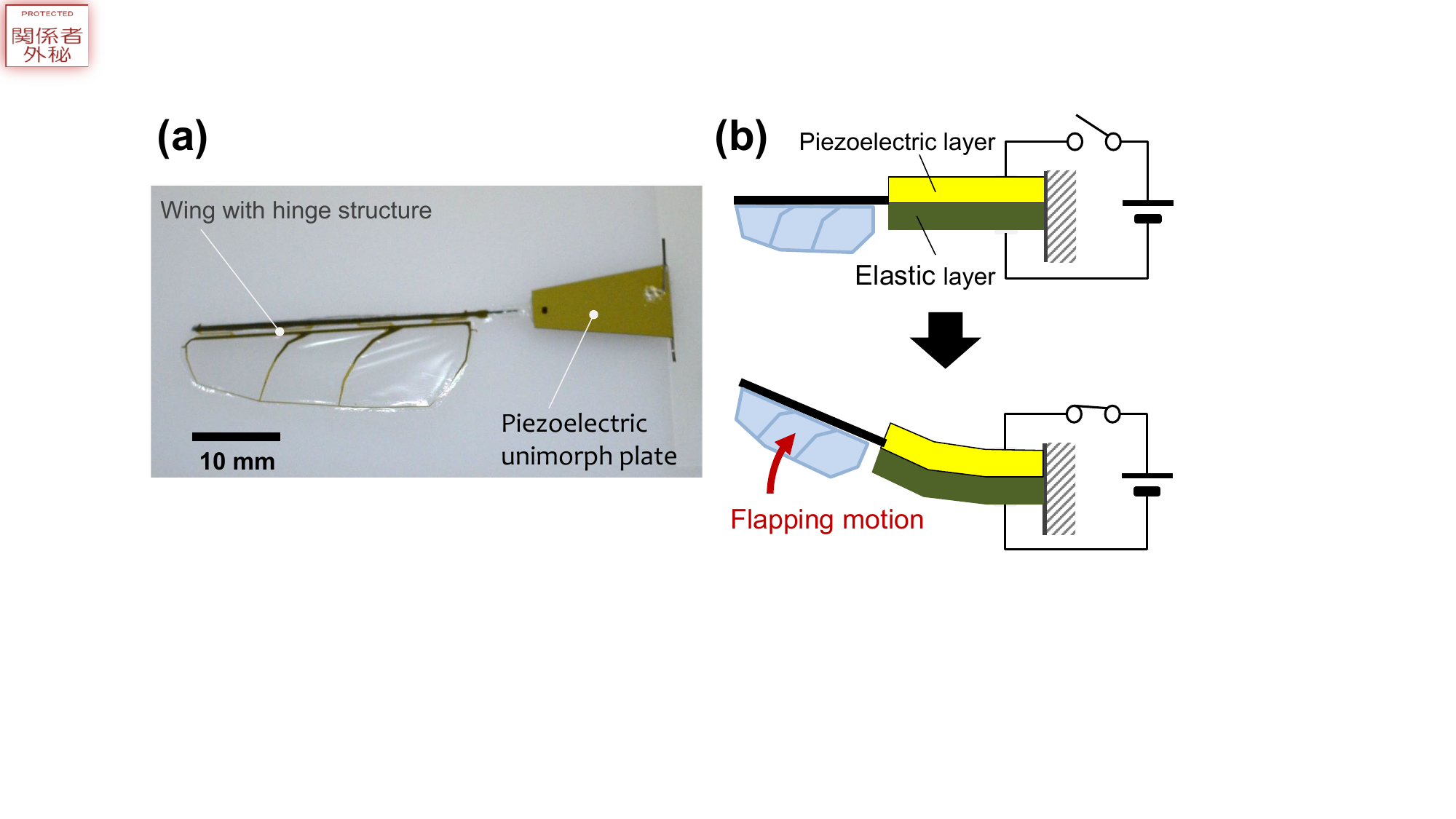}
\caption{
\textcolor{black}{
Direct-driven piezoelectric actuator. (a) Wing with a hinge structure attached to a piezoelectric unimorph plate. (b) Principle of the flapping motion.
}
\label{f:actuator}
}
\end{figure}

\textcolor{black}{
This study employs a direct-driven piezoelectric flapping-wing actuator, as shown in Fig. \ref{f:actuator}. The actuator consists of a wing directly attached to a piezoelectric unimorph plate without any mechanical transmission, resulting in a simple and lightweight structure. Passive wing rotation during the flapping stroke plays a key role in lift generation; therefore, a flexible hinge structure based on a polyimide torsional beam is placed near the top edge of the wing. To achieve large bending deformation, a single-crystal Pb(In\textsubscript{1/2}Nb\textsubscript{1/2})O\textsubscript{3}--Pb(Mg\textsubscript{1/3}Nb\textsubscript{2/3})O\textsubscript{3}--PbTiO\textsubscript{3} (PIN--PMN--PT) material is used for the unimorph actuator. When a sinusoidal voltage at an appropriate frequency is applied to the plate, the unimorph plate bends due to the piezoelectric effect, directly generating the flapping motion.
}

\textcolor{black}{
As shown in Fig. 1, two wings are paired to form a basic flapping unit in order to suppress coupling effects between the wings and the body, as well as between individual wings. This paired-wing configuration was investigated in our previous work (Jimbo et al., 2020) and is retained in the present robot as a fundamental structural element.
}

\textcolor{black}{
The flapping-wing actuator is driven by a simple direct-drive scheme in which the flapping amplitude is controlled solely by the applied voltage amplitude V. As reported in our previous study (Jimbo et al., 2020), both the stroke amplitude and the generated lift force reach their maximum values at a resonant frequency of approximately 115 Hz. When the actuator is driven at this resonant frequency, the lift force generated by a single paired wing increases with the applied voltage amplitude, reaching approximately 0.5 gf at V = 120 V. Based on this experimentally identified drive characteristic, the wing force is represented as a function of the voltage amplitude, 
\begin{equation}
    f_w = h(V),
    \label{eq:fV}
\end{equation}
which is adopted in the control-oriented modeling described in the following section.
}

\textcolor{black}{
Yaw-torque generation methods for flapping-wing robots commonly modify the speed ratio between the upstroke and downstroke of piezoelectric actuators by introducing a second-harmonic frequency component to the fundamental flapping signal \citep{2019Steinmeyer+,2021Chukewad+}. While such approaches have been experimentally shown to generate yaw torque, they require careful tuning of the second-harmonic component, including its relative magnitude and, in some cases, its phase, and are highly sensitive to manufacturing errors and parameter variations \citep{2021Chukewad+}. Moreover, as the drive signal becomes more complex, the associated signal-generation circuitry tends to increase in both weight and power consumption, which is undesirable for insect-scale, battery-powered platforms.
In contrast, the proposed design employs a direct-driven actuator operated at a fixed resonant frequency, where control is achieved solely through the voltage amplitude. This provides a simple and robust foundation for yaw-torque generation, which is formulated in the following modeling section.
}

\section{Modeling}

\subsection{
Tilted-wing arrangement and yaw torque generation
}

\begin{figure}[t]
\centering
\includegraphics[width=7.5cm]{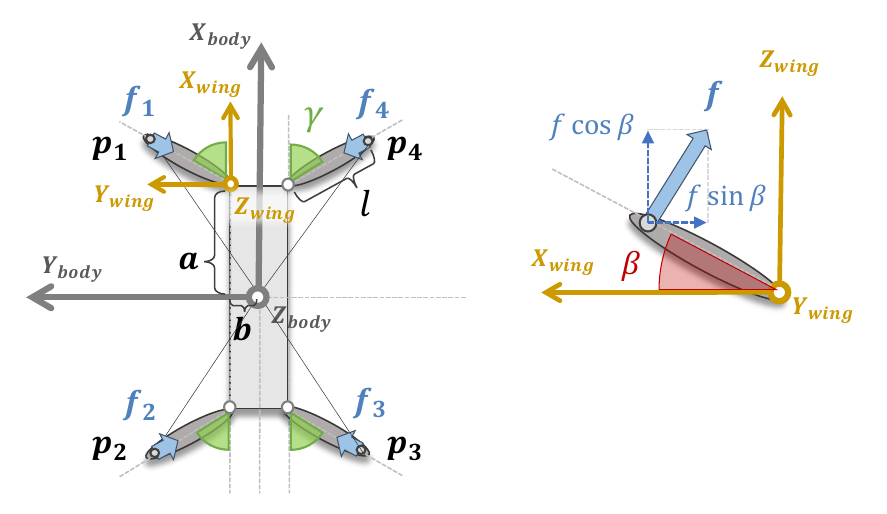}
\caption{
\textcolor{black}{
Wing arrangement and lift force direction
}
}
\label{f:WingArrangement}
\end{figure}

Fig.~\ref{f:WingArrangement} shows the wing arrangement of the newly developed four-paired-wing robot.
The wing arrangement is different from that of the three-paired-wing robot developed in our previous study \citep{2020_Jimbo}.
Furthermore, to generate yaw torque, the four paired wings are tilted slightly by an angle $\beta$, which is the rotation angle around the Y-axis of the wing coordinate system.

The body force $f(=[f_x, f_y,f_z]^\top)$ and torque ${\tau}(=[\tau_x, \tau_y,\tau_z]^\top)$ 
are given by
\begin{equation}
\begin{bmatrix}
f \\
\tau
\end{bmatrix} = M f_w, 
\label{eq:ftau_body}
\end{equation}
where ${f}_{w}=[f_1, f_2,  f_3,  f_4]^\top$ is the vector of wing forces generated by applying flapping amplitude $V$,
and
\begin{equation*}
M=
\begin{bmatrix}
M_1 \\
M_2
\end{bmatrix}
=
\begin{bmatrix}
{e}_1 & {e}_2 & {e}_3 & {e}_4 \\
{p}_1 \times {e}_1 &
{p}_2 \times {e}_2 &
{p}_3 \times {e}_3 &
{p}_4 \times {e}_4
\end{bmatrix} \in \mathbb{R}^{6\times 4}
\end{equation*}
is the mixing matrix.
Here, for the wing $i$,
\begin{eqnarray*}
&&e_i  =
\begin{bmatrix}
\cos \gamma_i \sin \beta_i \\
\sin \gamma_i \sin \beta_i \\
\cos \beta_i
\end{bmatrix},
p_i =
\begin{bmatrix}
a_i + l \cos \gamma_i \cos \beta_i \\
b_i - l \sin \gamma_i  \cos \beta_i \\
- l \sin \beta_i
\end{bmatrix}, 
\end{eqnarray*}
are the lengths of the body along the X- and Y-axes, respectively, 
$\beta_1 = \beta_2 = \beta_3 = \beta_4 = -\beta$, 
$\gamma_1=\gamma, \gamma_2=\pi-\gamma, \gamma_3=\pi + \gamma, \gamma_4=2\pi - \gamma$, 
$\gamma$ is the rotation angle around the Z-axis of the body coordinate system, and
$l$ is the distance between the mounting position of each wing and lift-force center.
The generated yaw torque $\tau_z$ is expressed as 
\begin{equation*}    
\tau_z = \sum_{i=1}^4 \left(a_i \sin \gamma_i - b_i \cos \gamma_i\right) f_i \sin \beta_i. 
\end{equation*}
Note that when $\beta_i = 0  (i=1,\cdots,4)$, $\tau_z=0$.


\subsection{
Representation of unknown offsets
}

The body force and torque in (\ref{eq:ftau_body}) are challenging to model accurately due to unmodeled aerodynamic effects, manufacturing errors, and time-dependent variations in system properties. Piezoelectric materials also exhibit characteristic variations, 
with lift force variations of approximately 10--20 $\%$ in the paired wings fabricated in \citep{2020_Jimbo}. 
A simple drive structure can mitigate these uncertainties, but they cannot be completely ignored. 
Therefore, the uncertainties in body force and torque are expressed as offsets.
The actual body force $f_{body} (=[f_{body,x}, f_{body,y}, f_{body,z}]^\top)$ and torque $\tau_{body}$ are given by
\begin{equation}
\begin{bmatrix}
{f}_{body} \\
{\tau}_{body}
\end{bmatrix}
=
\biggl({M} + \delta_M \biggr)
\biggl({f}_{w} + \delta_{f_w} \biggr)
=
\begin{bmatrix}
{f} - {f}_{o}\\
{\tau} - {\tau}_{o}
\end{bmatrix},
\label{eq:actualfandT}
\end{equation}
where
the offset force and torque are given by
\begin{equation}
\begin{bmatrix}
{f}_{o} \\
{\tau}_{o}
\end{bmatrix} \approx 
- \delta_M{f}_{w} 
- {M} \delta_{f_w}.
\label{eq:offset}
\end{equation}
They depend on $\delta_{f_w} \in \mathbb{R}^{4}$ and $\delta_M \in \mathbb{R}^{6 \times 4}$.
$\delta_{f_w}$ corresponds to the error of the wing force voltage modeling $f_w = h(V)$.
$\delta_M$ is the misalignment with respect to $\beta$, $\gamma$, and $l$.

\subsection{Flight dynamics}

Consider the rigid-body model of the robot shown in Fig.~\ref{f:WingArrangement}.
Using the actual force and torque, ${f}_{body}$ and ${\tau}_{body}$ in (\ref{eq:actualfandT}),
the dynamics are represented by
\begin{eqnarray}
\label{eq:Newton}
m \dot{{v}} &= & {R}{f}_{body} -m{g},  \\
\label{eq:Euler}
{J}\dot{{\omega}} &=& {\tau}_{body}-\left( {\omega}\times {J}{\omega} \right),
\end{eqnarray}
where $v(=[v_x, v_y, v_z]^\top)$ is the translational velocity of the body in the global coordinate system,
$\omega(=[\omega_{B,x}, \omega_{B,y}, \omega_{B,z}]^\top)$ is the angular velocity of the body in the body coordinate system,
$m$ and ${J}\in \mathbb{R}^{3\times 3}$ are the mass and inertia of the body, respectively,
${R}$ is the rotation matrix,
${g}=[0,0,g]^\top$, $g$ is the gravitational acceleration.

The relationship between $\omega$ and the attitude $\eta = [\phi, \theta, \psi]^\top$ of the body is described by
\begin{equation}
{\omega} = {G}\dot{{\eta}},
\label{eq:omega}
\end{equation}
where
\begin{equation}
{G} =
\begin{bmatrix}
1 & 0 & -\sin \theta \\
0 & \cos \phi & \cos \theta \sin \phi \\
0 & -\sin \phi & \cos \theta \cos \phi
\end{bmatrix}. \nonumber
\end{equation}

\subsection{Control-oriented formulation}

The translational dynamics on the horizontal plane is derived from (\ref{eq:Newton}), assuming the hovering state, $v_z \approx 0$, and $\omega_z  \approx 0$, to obtain
\begin{equation}
\begin{bmatrix}
\dot v_{x}^B \\
\dot v_{y}^B
\end{bmatrix}
=
\begin{bmatrix}
\theta \\
-\phi
\end{bmatrix}g,
\label{eq:xy}
\end{equation}
where $v_{x}^B$ and $v_{y}^B$ are the translational velocities in the body coordinate system.

When approximating the rotation matrix $R$ using  $\phi $ and $\theta$, the dynamics in the vertical direction, that is, along the Z-axis, is given by
\begin{equation}
m \ddot z = f_{z} - f_{o,z} - m g,
\label{eq:vz}
\end{equation}
where $z$ is the altitude, that is, the vertical position of the body.
Note that we assume $|\theta m g| \gg |f_{body,x}|$, $|\phi m g| \gg |f_{body,y}|$, and $|f_{body,z}| \gg |-\theta f_{body,x} +\phi f_{body,y}|$.

For the rotational dynamics, (\ref{eq:actualfandT}) and (\ref{eq:Euler})  are directly used to design the flight controller.

The transient response of the wing force to the flapping amplitude exhibits a lag in relation to the dynamics of the wing position. 
The lag can be approximated by the following first-order lag system:
\begin{eqnarray}
T\dot{f}_{z} &=& f_{d,z} - f_{z}, \label{eq:f-delay} \\
T\dot{{\tau}} &=& {\tau}_d - {\tau}, \label{eq:tau-delay}
\end{eqnarray}
where $T(>0)$ is a time constant and $f_{d,z} \in \mathbb{R}$ and ${\tau}_d \in \mathbb{R}^3$ are the demanded force and torque, respectively. Note that the lift force oscillates due to the reciprocating motion of the wing.
In this study, we focus on the steady lift force, ignoring high-frequency oscillations. 
While this simplification may limit control accuracy and design flexibility, it streamlines design, implementation, and maintenance while avoiding unexpected resonance.

\subsection{Controllability}
\label{sec:controllability}

From (\ref{eq:Euler}) and (\ref{eq:vz}), we consider the continuous-time linear time-invariant system 
\begin{equation}
    \dot \xi = A \xi + B f_w, 
    \label{eq:LTI}
\end{equation}
where $\xi = [m v_z, (J\omega)^\top]^\top \in \mathbb{R}^4$, $A={\mathbf 0}\in \mathbb{R}^{4\times 4}$, and $B=[M_1(3,:)^\top , M_2^\top]^\top\in \mathbb{R}^{4\times 4}$. This equation is weight-compensated and ignores the offset and centrifugal Coriolis forces.

For the driftless system of (\ref{eq:LTI}), the controllability Gramian at time $t_{\rm f}$ is expressed as
\begin{equation*}
    W_c(t_{\rm f}) = \int_0^{t_{\rm f}} {\rm e}^{At}BB^\top {\rm e}^{A^\top t} dt = t_{\rm f} B B^\top.
\end{equation*}
If $W_c(t_{\rm f})$, i.e. $BB^\top$ is a nonsingular matrix, a state transition to an arbitrary target can be realized at time $t_{\rm f}$. Then, the minimum input energy is expressed in quadratic form using the target state and the inverse of $W_c(t_{\rm f})$ \citep{Kailath1980}. 
As $BB^\top$ is singular for $B\in \mathbb{R}^{4\times 3}$ for the robot with three paired wings in \citep{2020_Jimbo}, controlling four states of $\xi$ simultaneously requires a robot with four paired wings.

\section{Flight controller}

The offset force and torque in (\ref{eq:offset}) are unknown parameters.
They are uncertain and cannot be measured.
Furthermore, a change in the relationship between $f_w$ and $V$ occurs due to damage during use.
Therefore, an adaptive controller is adopted in this study.

\subsection{Velocity control}

To track the velocities $v_x^B$ and $v_y^B$ on the horizontal plane to the target $v_{x,d}^B$ and $v_{y,d}^B$, respectively,
considering (\ref{eq:xy}),
the corresponding target angles of the pitch $\theta_d$ and roll $\phi_d$ are set to
\begin{eqnarray*}
\theta_d &=& -h_x \left(v_x^B - v_{x,d}^B \right)/g, \label{eq:theta_d} \\
\phi_d &=& h_y \left(v_y^B - v_{y,d}^B \right)/g, \label{eq:phi_d}
\end{eqnarray*}
where $h_x$ and $h_y$ are positive constants.

\subsection{Attitude control}

To track the attitude ${\eta}$ to the target ${\eta}_d(= [\phi_d, \theta_d, \psi_d]^\top)$,
from (\ref{eq:omega}),
the target angular velocity of the robot $\omega$ and higher-order derivatives are designed as follows:
\begin{equation}
{\omega}_d = - G {K}_\eta {e}_\eta, \ \ \ 
\dot{{\omega}}_d \approx -G {K}_\eta \dot{{\eta}}, \ \ \ 
\ddot{{\omega}}_d \approx -G {K}_\eta \ddot{{\eta}}
\label{eq:omegad}
\end{equation}
where ${e}_\eta = {\eta} - {\eta}_d$,
${K}_\eta(={\rm diag}(k_{\eta1}, k_{\eta2}, k_{\eta3}))$ is a positive diagonal matrix.

To track the angular velocity ${\omega}$ to the target ${\omega}_d$,
from (\ref{eq:Euler}) and (\ref{eq:tau-delay}),
the demanded torque ${\tau}_d$ and the adaptive law to estimate the offset torque $\hat{{\tau}}_o$ are designed as 
\begin{eqnarray}
{\tau}_d &=& - {K}_\omega {s}_\omega
+ J\dot{{\omega}}_r  + {F}
     + T \left(J \ddot{{\omega}}_r  + \dot{{F}}\right)
    + \hat{{\tau}}_o,
\label{eq:adapt_tau}
\\
\dot{\hat{{\tau}}}_o &=& - {\Gamma}_\omega {s}_\omega,
\label{eq:adapt_tauo}
\end{eqnarray}
where $F(=\omega \times  J \omega)$ is the centrifugal Coriolis force,
\begin{eqnarray}
{s}_\omega           &=& \dot{{e}}_\omega +{\Lambda}_\omega {e}_\omega = \dot{{\omega}} - \dot{{\omega}}_r, \label{eq:s_omega}\\
{e}_\omega           &=&{\omega}-{\omega}_d,  
\nonumber \\
\dot{{\omega}}_r &=& \dot{{\omega}}_d - {\Lambda}_\omega ({\omega} -{\omega}_d), \nonumber
\end{eqnarray}
${\Lambda}_\omega (= {\rm diag}(\lambda_{\omega1}, \lambda_{\omega2}, \lambda_{\omega3}))$,
${K}_\omega  (= {\rm diag}(k_{\omega1}, k_{\omega2}, k_{\omega3}))$, and
${\Gamma}_\omega (= {\rm diag}(\gamma_{\omega1}, \gamma_{\omega2}, \gamma_{\omega3}))$
are positive diagonal matrices.

\begin{theorem}
\label{thm:attitude}
On applying the control input of (\ref{eq:adapt_tau}) and the adaptive law of (\ref{eq:adapt_tauo}) with (\ref{eq:s_omega}) to (\ref{eq:Euler}) and (\ref{eq:tau-delay}), the sliding variable $s_\omega$ satisfies $s_\omega \rightarrow 0$ as $t \rightarrow \infty$ 
\textcolor{black}{
, under the assumption that the unknown offset torque $\tau_o$ is constant.
}
\end{theorem}
\begin{proof}
See Appendix A for the proof. 
\end{proof}
From Theorem \ref{thm:attitude}, considering (\ref{eq:s_omega}), $\omega$ asymptotically converges to $\omega_d$ as $t \rightarrow \infty$.
Similarly, from (\ref{eq:omega}) and (\ref{eq:omegad}), $\eta$ asymptotically converges to $\eta_d$ as $t \rightarrow \infty$.
%

%
\subsection{Vertical control}
To control the flapping-wing robot in the vertical direction, 
from (\ref{eq:vz}) and (\ref{eq:f-delay}),
the demanded force $f_{d,z}$ and the adaptive law to estimate the offset force $\hat{f}_{o,z}$ are designed as
\begin{eqnarray}
f_{d,z}/m &=& - k_z s_z + \ddot{z}_r + T \dddot{z}_r + g + \hat{f}_{o,z}/m,  \label{eq:adapt_f} \\
\dot{\hat{f}}_{o,z} &=& - \gamma_z s_z/m, \label{eq:adapt_fo}
\end{eqnarray}
where, for the vertical velocity control to track $\dot{z}$ to the target $\dot{z}_d$, 
\begin{eqnarray}
s_z &=& \ddot{e}_z + \lambda_z \dot{e}_z = \ddot{z} - \ddot{z}_r, \label{eq:s_vz} \\
\ddot{z}_r &=&  - \lambda_z (\dot{z}-\dot{z}_d), 
\nonumber
\end{eqnarray}
and for the vertical position (altitude) control to track $z$ to the target $z_d$,  
\begin{eqnarray}
s_z &=& \ddot{e}_z + 2\lambda_z \dot{e}_z + \lambda_z^2 e_z = \ddot{z} - \ddot{z}_r, \label{eq:s_z} \\
\ddot{z}_r &=&  - 2\lambda_z \dot{z} -\lambda_z^2(z-z_d). \nonumber
\end{eqnarray}
Here, $e_z = z - z_d$. 
$\lambda_z$, $k_z$, and $\gamma_z$ are positive constants.

\begin{theorem}
\label{thm:vertical}
On applying the control input of (\ref{eq:adapt_f}) and the adaptive law of (\ref{eq:adapt_fo}) with (\ref{eq:s_vz}) or (\ref{eq:s_z}) to (\ref{eq:vz}) and (\ref{eq:f-delay}), the sliding variable $s_z$ satisfies $s_z \rightarrow 0$ as $t \rightarrow \infty$
\textcolor{black}{
, under the assumption that the unknown offset force $f_{o,z}$ is constant.
}
\end{theorem}
\begin{proof}
See Appendix B for the proof. 
\end{proof}
From Theorem \ref{thm:vertical}, considering (\ref{eq:s_z}), $z$ asymptotically converges to $z_d$ as $t \rightarrow \infty$.
\subsection{Lift-force demand and flapping-amplitude control}
From the demanded torque $\tau_d$ and force $f_d$ in (\ref{eq:adapt_tau}) and (\ref{eq:adapt_f}) as well as the mixing matrix in (\ref{eq:ftau_body}), the lift-force demand is given by
\begin{equation}
{f}_{w,d}
=\begin{bmatrix}
{M}_1(3,:) \\
{M}_2
\end{bmatrix}^{-1}
\begin{bmatrix}
{f}_{d,z} \\
\boldsymbol{\tau}_d
\end{bmatrix}.
\label{eq:mixing_inverse}
\end{equation}

\textcolor{black}{
The flapping-wing actuator used in this study is driven by simply changing the flapping amplitude, as described in Section~2 \citep{2018_1st_Ozaki,2020_Jimbo}.
The relationship between the wing force $f_w$ and the flapping amplitude $V$ in \eqref{eq:fV} is modeled around the voltage at which the robot supports its own weight. 
}

%

%
As a result, the required flapping amplitude $V_{d}\in\mathbb{R}^4$  is derived from (\ref{eq:mixing_inverse}) and (\ref{eq:fV}) as $V_d = h^{-1}(f_{w,d})$. 
Then, $V_d$ is multiplied by a sinusoidal wave and then applied to the robot. Note that this can be performed using a digital circuit,
such as a pulse-width modulator or similar modulator,
which is much simpler than a circuit that generates complex analog waveforms \citep{2023_5th_Ozaki}.
%

\section{Simulation}

\textcolor{black}{
The effectiveness of the proposed controller was validated 
through numerical simulations under unknown offsets. 
Before conducting the performance evaluations, 
the controllability of the flapping-wing robot fabricated 
with the tilting angles shown in Fig.~\ref{f:flappingwingRobots} 
was examined using the controllability analysis described in 
Section~\ref{sec:controllability}. 
The proposed controller was then compared with a conventional 
LQI controller derived from the linearization of 
(\ref{eq:Newton}) and (\ref{eq:Euler}), including the first-order 
lag of the wing forces.
}
\textcolor{black}{
It should be noted that the LQI controller is designed based on a linearized model around a hovering equilibrium without offset terms. Accordingly, the simulation results using the LQI controller are intended to evaluate hovering performance under idealized conditions, rather than to represent a controller that is directly implementable on the hardware platform.
}

\subsection{Controllability analysis}

\begin{figure}[t!]
\vspace{-5mm}
\centering
\begin{subfigure}[b]{0.90\linewidth}
\centering
\includegraphics[width=8cm]{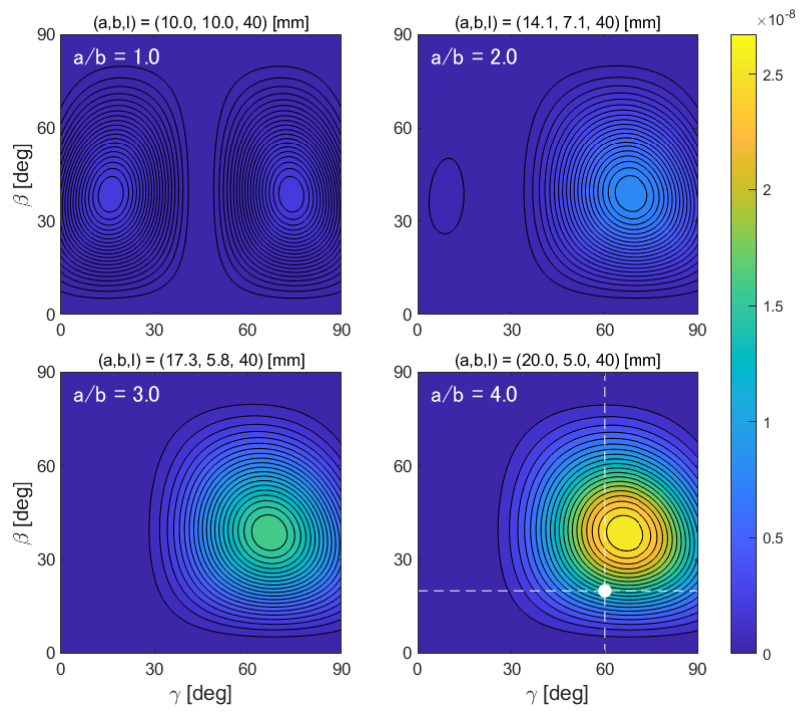}
\vspace{-5mm}
\caption{Determinant}
\label{f:detK}
\end{subfigure}
\centering
\begin{subfigure}[b]{0.90\linewidth}
\centering
\includegraphics[width=8cm]{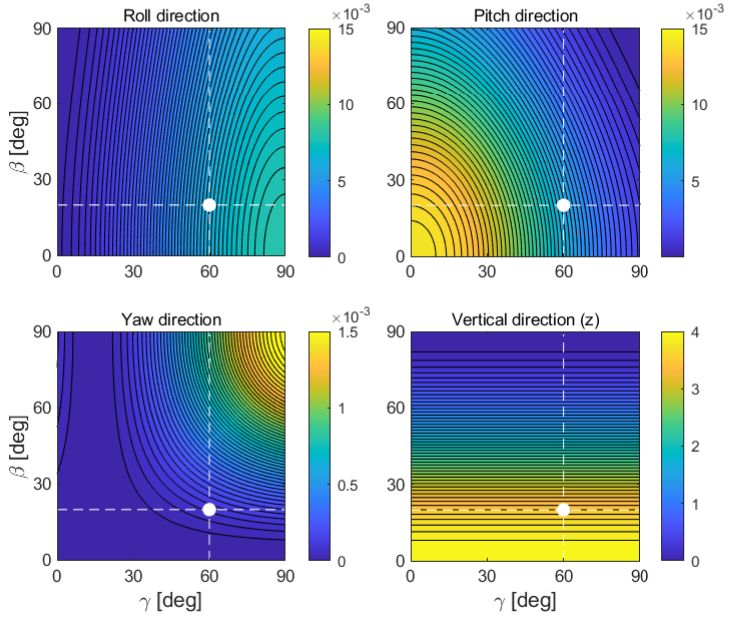}
\caption{Eigenvalues when $(a,b,l)=(20, 5, 40)$ mm.}
\label{f:eigK}
\end{subfigure}
\caption{
\textcolor{black}{
Controllability Gramian $BB^\top$
}
}
\end{figure}

%
Fig.~\ref{f:detK} shows the determinant of $BB^\top$ for varying $\beta$, $\gamma$, and the ratio of body lengths, $a/b$, at $ab=1$ 
\textcolor{black}{
$\rm{cm^2}$ and $l=40$ mm.
}
A larger value of $a/b$ corresponds to a larger determinant of $BB^\top$ and a more controllable body. 
In this study, we set $(a,b)=(20,5)$ mm because $b=5$ mm is the smallest value feasible for manufacturing.

\textcolor{black}{
A single paired wing generates approximately $0.5$ gf at $120$ V \citep{2020_Jimbo}, giving a total of $2.0$ gf for four pairs.
The determinant peaks at $(\beta,\gamma)=(38,66)$ degrees and $a/b=4$ (Fig.~\ref{f:detK}), where the effective lift force becomes $\cos\beta \times 2.0 = 0.79\times2.0 = 1.58$ gf, corresponding to a lift-to-weight ratio of $1.04$ for the $1.52$~g robot.
This result indicates that determinant maximization enhances controllability but leaves little lift margin. 
In contrast, the design setting of $\beta = 20$~deg (white circle in Fig.~\ref{f:detK}) yields a lift of approximately $1.88$ gf, corresponding to a lift-to-weight ratio of $1.24$. 
This ratio lies within a reasonable range, considering that the previous three-pair-wing robot achieved about $1.3$, and thus the design can be regarded as appropriate in terms of lift margin. 
}

\textcolor{black}{
Directions in the state space corresponding to large eigenvalues of $BB^\top\textcolor{black}{\in\mathbb{R}^{4\times4}}$ can be controlled with minimal input energy.
\textcolor{black}{
Fig.~\ref{f:eigK} shows the eigenvalues of $BB^\top$ for $(a,b,l) = (20,5,40)$ mm.
The correspondence to the roll, pitch, yaw, and vertical directions was determined from the associated eigenvectors.
}
It can be observed that the vertical direction is the most controllable, while the yaw direction is the least controllable.
In addition, prioritizing the yaw direction reduces the controllability in the vertical and pitch directions.
For the selected value of $\gamma = 60$~deg (white circle in Fig.~\ref{f:eigK}), the roll and pitch directions are well balanced, although the controllability in the yaw direction remains relatively weak.
}

\begin{table}[t!]
\begin{center}
\caption{
Plant Parameters
\label{t:plant}
}
\begin{tabular}{llll}
\hline
Name & Symbol & Value & Unit \\
\hline \hline
Body mass & $m$  & \textcolor{black}{$1.52*10^{-3}$} & $\rm kg$\\
Body inertia  & $J_1$ & $1.50*10^{-7}$ & $\rm kg \cdot m^2$ \\
$(J={\rm diag}(J_1, J_2, J_3))$ & $J_2$ & $1.35*10^{-7}$ & $\rm kg \cdot m^2$ \\
 & $J_3$ & $2.21*10^{-7}$ & $\rm kg \cdot m^2$ \\
Body lengths & $a$ & $20.0*10^{-3}$ & $\rm m$ \\
 & $b$ & $5.0*10^{-3}$ & $\rm m$ \\
\hline
Wing angles &
 $\beta$ & $20*\pi/180$ & $\rm rad$ \\
 & $\gamma$ & $60*\pi/180$ & $\rm rad$ \\
 Length between lift \\ center and wing root
 & $l$ & $40.0*10^{-3}$ & $\rm m$ \\
\hline
Time constant of \\ first-order lag & $T$ & $0.013$ & sec \\
\hline
Gravitational \\ acceleration & $g$  & $9.81$ & $\rm kg\cdot m/sec^2$ \\
\hline
\end{tabular}
\end{center}
\begin{center}
\caption{Control Parameters}
\vspace{3mm}
\label{t:parameters_Adapt}
\begin{tabular}{llll}
\hline
 & Symbol & Value & Unit \\
\hline \hline
Attitude
        & $h_x,h_y$ & $1/0.5$ & $\rm 1/sec$ \\
        & $k_{\eta1},k_{\eta2},k_{\eta3}$ & $1/0.1$ & $\rm 1/sec$ \\
        \cline{2-4}
	& $\lambda_{\omega1},\lambda_{\omega2},\lambda_{\omega3}$ & $1/0.1$ & $\rm 1/sec$ \\
		& $k_{\omega1}$ & $9.50*10^{-8}$ & - \\
		& $k_{\omega2}$ & $8.55*10^{-8}$ & - \\
		& $k_{\omega3}$ & $1.40*10^{-7}$ & - \\
		& $\gamma_{\omega1}$ & $7.70*10^{-6}$ & - \\
		& $\gamma_{\omega2}$ & $6.93*10^{-6}$ & - \\
		& $\gamma_{\omega3}$ & $1.13*10^{-5}$ & - \\
\cline{2-4}
		& $\hat{\tau}_{o,x}$,$\hat{\tau}_{o,y}$,$\hat{\tau}_{o,z}(t=0)$ & $0$ & ${\rm Nm}$ \\
\hline
Altitude
		& $\lambda_z$ & $1/0.5$ &  $\rm 1/sec$ \\
		& $k_z$ & $6.34*10^{-1}$ & - \\
		& $\gamma_{z}$ & $2.05*10^{-4}$ & - \\
		 \cline{2-4}
         & $\hat{f}_{o,z}(t=0)$ & $0$ & ${\rm N}$ \\
\hline
\end{tabular}
\end{center}
\begin{center}
\caption{
Unknown Offsets
}
\label{t:offset}
\begin{tabular}{llll}
\hline
& Symbol & Value & Unit \\
\hline \hline
& $\delta \beta$ & ${10}*\pi/180$ & $\rm rad$ \\
Case 1 & $\delta \gamma$ & ${10}*\pi/180$  & $\rm rad$ \\
& $\delta l$ & $5*10^{-3}$ & $\rm m$ \\
\hline
Cases 2 and 3 & ${\delta f}_w$ & $m g/4/{3}*[0, -1, 0, 0]^\top$ & $\rm  N$ \\
\hline
\end{tabular}
\end{center}
\end{table}
%
\subsection{Simulation setup}
Tables \ref{t:plant} and \ref{t:parameters_Adapt} present the parameters of the plant and the proposed controller, respectively, used in the numerical simulations. 
Note that the parameters, particularly the value of the body inertia $J$, may differ from the actual values.
\textcolor{black}{
The weight is set to 1.52 g, which is identical to that of the actual robot shown in Fig.~\ref{f:flappingwingRobots}, in order to ensure consistency between the simulation and the flight experiment.
}
%
Note that the control parameters were adjusted to obtain similar closed-loop responses from both controllers.

\subsection{Case studies}
In the simulations, we consider some case studies with lift force offsets, as shown in Table \ref{t:offset}.
\begin{description}
\item[Case 1. Offset due to manufacturing errors] \mbox{}\\
The misalignments $\delta \beta$ and $\delta \gamma$ that occur during manufacturing were both set to 10 degrees. 
The error $\delta l$ in the distance $l$ from the mounting position of each wing to 
lift-force center was set to 5 mm for each wing.
\item[Case 2. Offset due to modeling error] \mbox{}\\
There exists an error ${\delta f}_w$ in the lift-force model of (\ref{eq:fV}).
In this case, only the second wing has a modeling error of one-third of $mg/4$, which is necessary when four paired wings share the weight.
\item[
Case 3. After adapting for modeling errors
]
\mbox{}\\
The numerical experiment for Case 2 was performed again. However, the offset torque $\hat{\tau}_o$ and force $\hat{f}_{o,z}$, estimated previously in Case 2, were used as the initial values of the parameter adaptation laws in equations (\ref{eq:adapt_tauo}) and (\ref{eq:adapt_fo}). 
\end{description}

\subsection{Comparison results}

\begin{figure}[!b]
\centering
  \begin{subfigure}[b]{1.00\linewidth}
      \centering 
      \includegraphics[width=9cm]{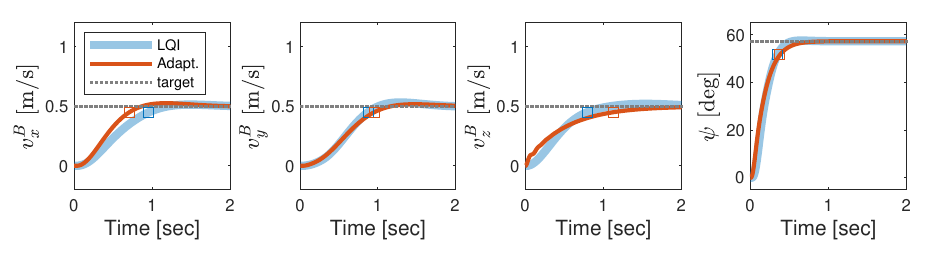}
      \vspace{-2mm}
      \caption{No lift force offset} 
      \label{f:NoOffset_VelAtt}      
  \end{subfigure}  
  \begin{subfigure}[b]{1.00\linewidth}
       \centering 
       \includegraphics[width=9cm]{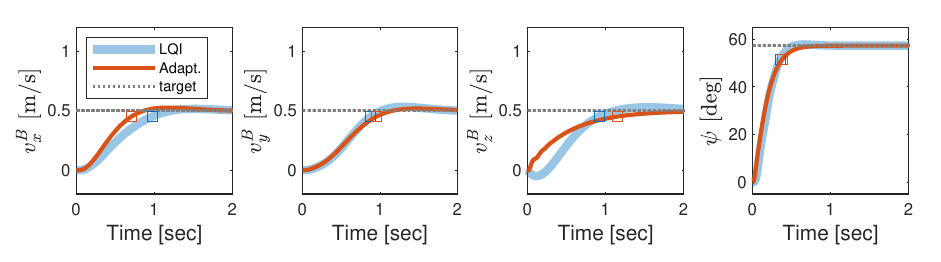}
       \vspace{-2mm}
       \caption{Case 1} \label{f:ProdErr_10deg_VelAtt}
  \end{subfigure}
  \begin{subfigure}[b]{1.00\linewidth}
       \centering 
       \includegraphics[width=9cm]{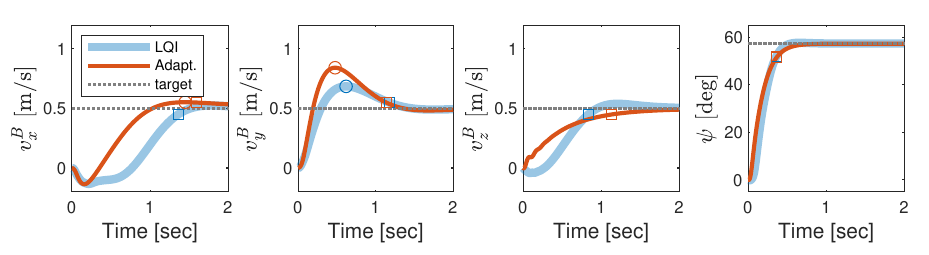}
       \vspace{-2mm}
       \caption{Case 2}
       \label{f:MapErr_df1by3_VelAtt}
  \end{subfigure}
  \begin{subfigure}[b]{1.00\linewidth}
       \centering 
       \includegraphics[width=9cm]{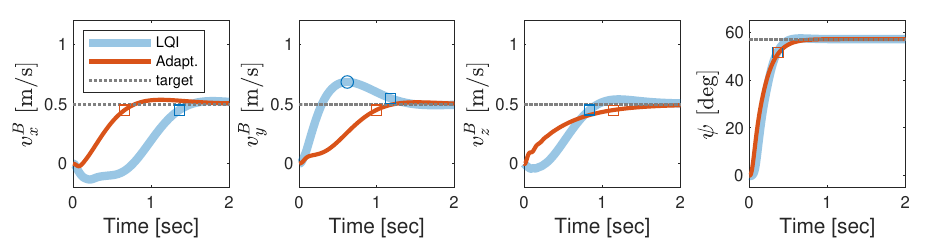}
       \vspace{-2mm}
       \caption{Case 3}
       \label{f:MapErr_df1by3_HotStart_VelAtt}
  \end{subfigure}
\caption{
\textcolor{black}{
Step responses. The circles and squares represent overshoot of $10 \%$ or greater, and settling times, respectively.
}
}
\label{f:StepResponse}
\end{figure}

\begin{table}[!t]
\textcolor{black}{
\centering
\caption{
\textcolor{black}{
Performance metrics for step responses.
OS, ST, and SS RMSE denote the overshoot, settling time, and steady-state root mean square error evaluated over 1.5--2.0 s, respectively.
}
}
\label{tab:performance}
\normalsize
\setlength{\tabcolsep}{5.0pt}
\begin{tabular}{p{0.8cm}lcccccc}
\hline
\multirow{2}{*}{Case} & \multirow{2}{*}{Output}
& \multicolumn{2}{c}{OS [\%]}
& \multicolumn{2}{c}{ST [sec]}
& \multicolumn{2}{c}{SS RMSE} \\
 &  & LQI & Adp. & LQI & Adp. & LQI & Adp. \\
\hline\hline
\multirow{4}{*}{
\makecell[l]{No-\\offset}
}
& $v_{x}^{B}$ [m/s] & 3  & 5  & 1.0 & 0.7 & 0.01 & 0.01 \\
& $v_{y}^{B}$ [m/s] & 6  & 3  & 0.9 & 1.0 & 0.01 & 0.01 \\
& $v_{z}^{B}$ [m/s] & 4  & 0  & 0.8 & 1.1 & 0.02 & 0.01 \\
& $\psi$ [deg]         & 1  & 0  & 0.3 & 0.4 & 0.01 & 0.03 \\
\hline
\multirow{4}{*}{Case 1}
& $v_{x}^{B}$ [m/s] & 3  & 5  & 1.0 & 0.7 & 0.01 & 0.01 \\
& $v_{y}^{B}$ [m/s] & 6  & 4  & 0.9 & 1.0 & 0.01 & 0.01 \\
& $v_{z}^{B}$ [m/s] & 6  & 0  & 0.9 & 1.1 & 0.02 & 0.02 \\
& $\psi$ [deg]         & 1  & 0  & 0.4 & 0.4 & 0.00 & 0.02 \\
\hline
\multirow{4}{*}{Case 2}
& $v_{x}^{B}$ [m/s] & 5  & 10 & 1.4 & 1.6 & 0.02 & 0.04 \\
& $v_{y}^{B}$ [m/s] & 37 & 68 & 1.2 & 1.1 & 0.01 & 0.01 \\
& $v_{z}^{B}$ [m/s] & 8  & 0  & 0.8 & 1.1 & 0.01 & 0.02 \\
& $\psi$ [deg]         & 1  & 0  & 0.4 & 0.4 & 0.02 & 0.08 \\
\hline
\multirow{4}{*}{Case 3}
& $v_{x}^{B}$ [m/s] & 5  & 7  & 1.4 & 0.7 & 0.02 & 0.01 \\
& $v_{y}^{B}$ [m/s] & 37 & 3  & 1.2 & 1.0 & 0.01 & 0.01 \\
& $v_{z}^{B}$ [m/s] & 8  & 0  & 0.8 & 1.1 & 0.01 & 0.02 \\
& $\psi$ [deg]         & 1  & 0  & 0.4 & 0.4 & 0.02 & 0.01 \\
\hline
\end{tabular}
\label{t:PerformanceMetrics}
}
\end{table}

\begin{figure}[t!]
\centering
\begin{subfigure}[b]{0.5\linewidth}
    \centering
        \includegraphics[width=4.2cm]
        {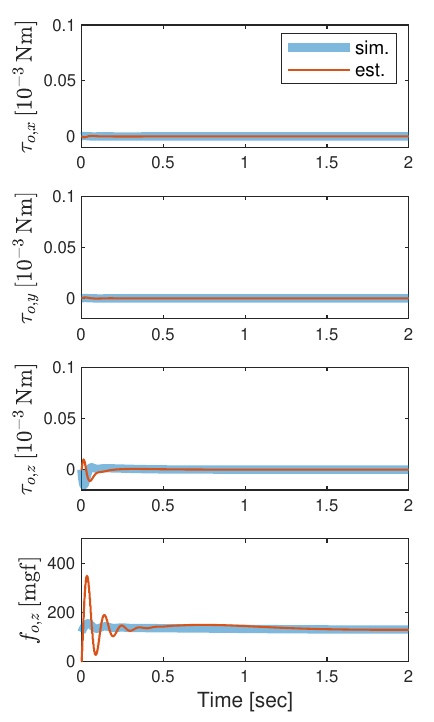}
        \caption{Case 1}
        \label{f:offset1}
    \end{subfigure}
    \hspace{-0.05\linewidth}
\begin{subfigure}[b]{0.5\linewidth}
    \centering
        \includegraphics[width=4.2cm,height=7.3cm]
        {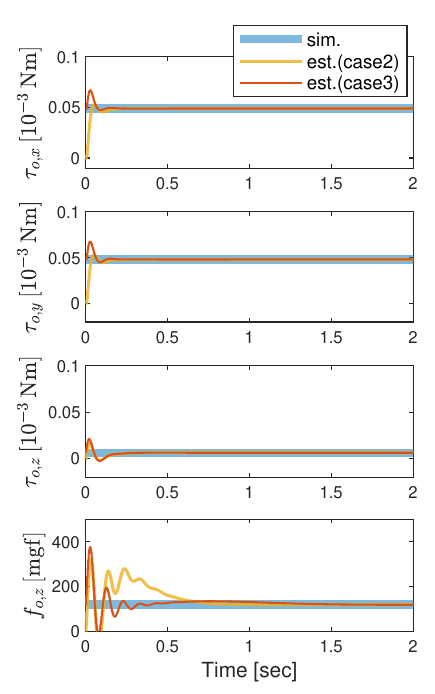}
        \caption{Cases 2 and 3}
        \label{f:offset23}
    \end{subfigure}    
\caption{
\textcolor{black}{
Offsets and the estimates obtained from the proposed controller
}
\label{f:offset}
}
\end{figure}

Fig. \ref{f:StepResponse} and Table \ref{t:PerformanceMetrics} show the step responses of both controllers and the performance metrics  (overshoot, delay time, and settling time) in four cases, including those without offsets. Here, a target value of $\boldsymbol{r}=[v_{x,d}^B, v_{y,d}^B, v_{z,d}^B, \psi_d]^\top = [0.5, 0.5, 0.5, 1]^\top$ was set, and the settling time threshold was defined as $\pm 10 \%$.

Fig.~\ref{f:NoOffset_VelAtt} presents the simulation result without the offsets in Table \ref{t:offset}. As shown, the velocity and attitude controlled by each controller converged to the targets in approximately $1$ s without significant overshoots.
The metrics in the no-offset case, shown in Table~\ref{t:PerformanceMetrics} were almost identical.

Second, Fig.~\ref{f:ProdErr_10deg_VelAtt} shows the result for Case 1 with the misalignments shown in Table \ref{t:offset}.
The control performance of the LQI controller was worse than that without the offset shown in Fig.~\ref{f:NoOffset_VelAtt}.
The reverse response of $v_z$ occurred for the LQI controller. 
In contrast, the proposed controller, 
as shown in Fig.~\ref{f:offset1},  estimated the unknown offset earlier, resulting in performance metrics in Table~\ref{t:PerformanceMetrics} that were nearly identical for both the no offset and Case 1.  

Third, Fig.~\ref{f:MapErr_df1by3_VelAtt} shows the result for Case 2 with the wing force voltage modeling error shown in Table \ref{t:offset}.
The control performance of the LQI controller was worse than that without the offset shown in Fig.~\ref{f:NoOffset_VelAtt}.
The reverse responses of $v_x$ and $v_z$ occurred for the LQI controller.
$v_y$ had an overshoot.
In contrast, for the proposed controller, the response of $v_z$ was the same as that in Fig.~\ref{f:NoOffset_VelAtt}. 
However, a reverse response of $v_x$ occurred, and $v_y$ had a larger overshoot than that of the LQI controller 
as shown in Table~\ref{t:PerformanceMetrics}. 
Fig.~\ref{f:offset23} shows the offset for Case 2. The larger overshoot is due to the fact that the estimation of the offset force $f_{o,z}$ took approximately $0.7$ s.
%

%
%

Finally, Fig.~\ref{f:MapErr_df1by3_HotStart_VelAtt} shows the result for Case 3 using the proposed controller with the parameters adapted for modeling errors.
\textcolor{black}{
The response of the proposed controller was neither a reverse response nor an overshoot, in contrast to the LQI controller.
Additionally, as shown in Table~\ref{t:PerformanceMetrics}, the control performance was clearly improved compared to that shown in Fig.~\ref{f:MapErr_df1by3_VelAtt}.
This improvement is attributed to the rapid convergence of the estimated offset force $f_{o,z}$, which converges within approximately $0.1$ s, as shown in Fig.~\ref{f:offset23}.
In addition, it is observed that, in all simulation cases (Cases 1–3), the LQI controller exhibits a pronounced transient undershoot in the vertical response when offset uncertainties are present.
}

%
\section{Experiment}
The effectiveness of the proposed controller was demonstrated through a flight experiment.  
\textcolor{black}{
The weight of the flapping-wing robot was set to $1.52$ g, which is identical to that of the robot shown in Fig.~\ref{f:flappingwingRobots} and consistent with the simulation conditions in Table~\ref{t:plant}.
}

%

%
\subsection{Experimental setup}
Our experimental system is depicted in Fig.~\ref{f:exp_sys}. 
This system tracks markers attached to the body using 
the OptiTrack Prime 17W motion capture system ($1.7$ MP, $70$-deg field of view, set at $333$ fps, NaturalPoint) 
and a computer (Intel Core i9-9900K, $8$-core $3.6$ GHz, $64$ GB of RAM), which calculates the position and orientation of the robot. 
These values are then sent to a control computer (Intel Core i7-7700K, $4$-core $4.2$ GHz, $32$ GB of RAM) within $3$ ms.
The proposed controller calculates the flapping amplitude, which is then multiplied by a sinusoidal wave at 115 Hz, generated by a function generator (Precision 4050B, B\&K), using a multiplier (AD633, Analog Devices). The amplitude is further amplified 30 times using an amplifier (HJPZ-0.3P×3, Matsusada Precision) and then applied to the robot through enameled wires. Note that the same flapping amplitude is applied to paired wings to generate an equal force. 

\textcolor{black}{
In the experiments, the proposed adaptive controller is activated after a short open-loop liftoff phase of 0.12 s. During this phase, a constant voltage is applied to lift the robot slightly from the ground, resulting in a vertical displacement of 0.78 cm, which is negligible compared to the 110 mm wingspan. No feedback control, parameter estimation, or alignment is performed during this phase. Its sole purpose is to exclude the effects of initial contact with the environment before engaging the adaptive controller. 
}

\begin{figure}[t!]
\vspace{-4mm}
  \begin{center}
    \includegraphics[width=8cm]{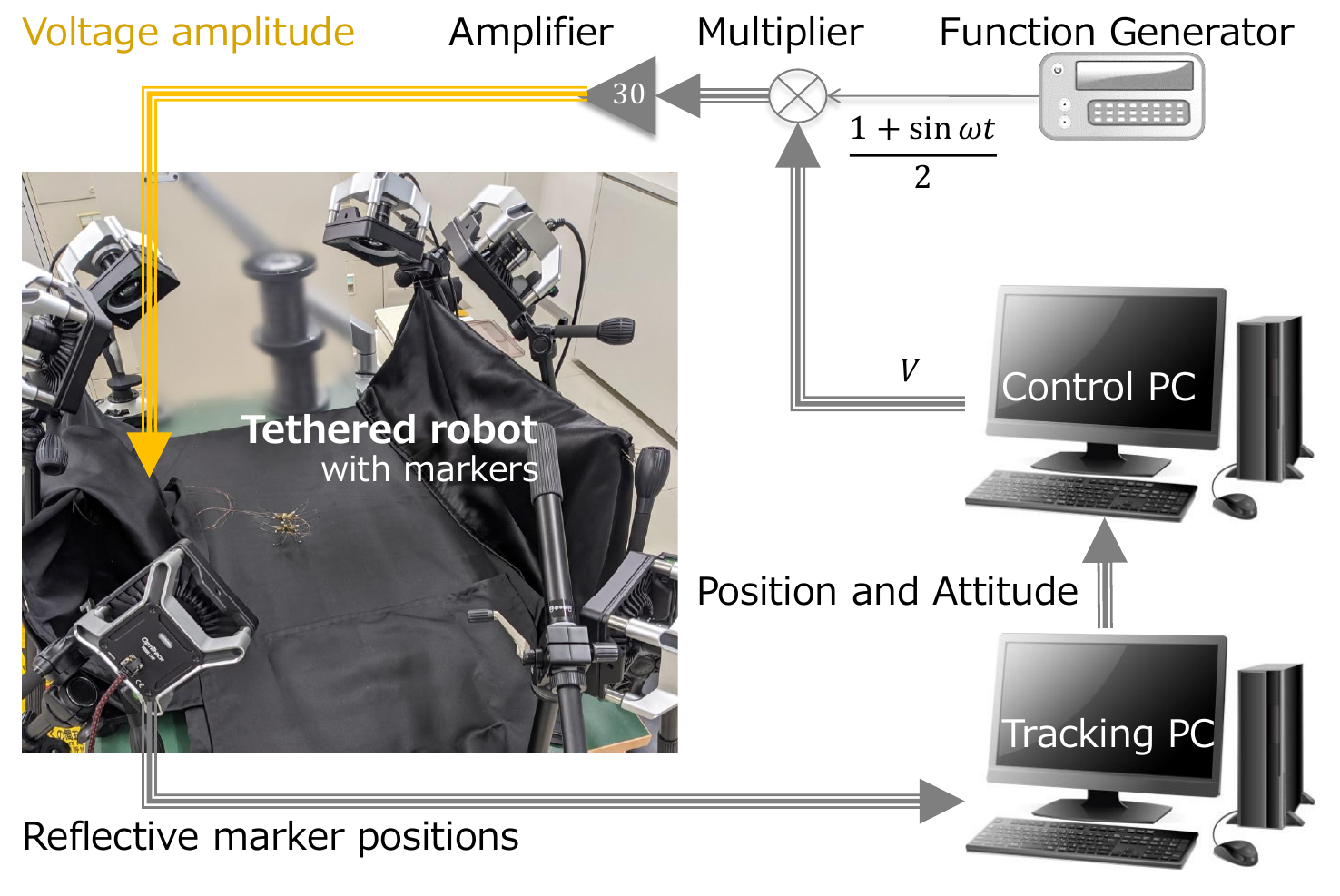}
    \caption{Experimental system for tethered controlled flight}
    \label{f:exp_sys}
  \end{center}
\end{figure}

\begin{figure}[t!]
\centering
  \begin{subfigure}[b]{0.85\linewidth}
      \centering
      \includegraphics[width=8cm]{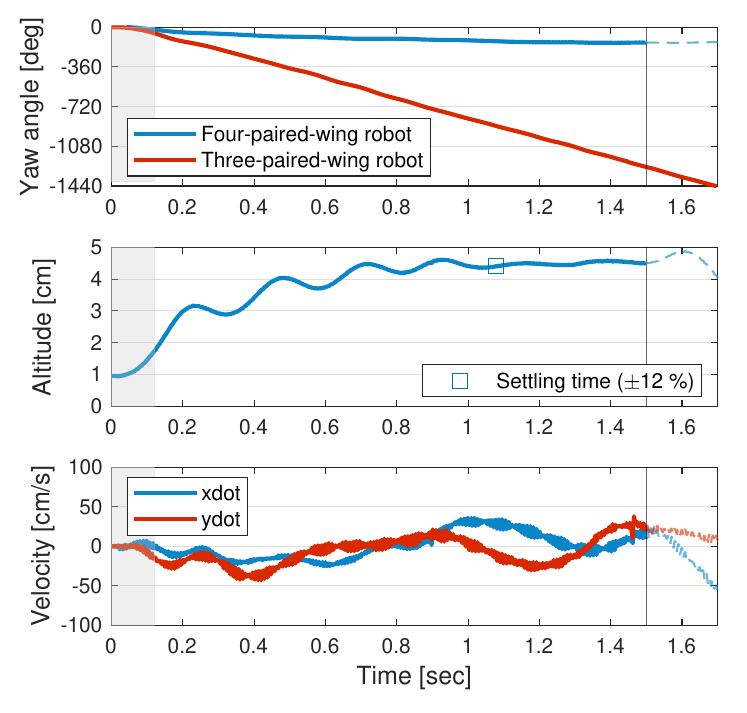}
      \caption{
      \textcolor{black}{
Time histories of yaw angle, altitude, and horizontal velocity during flight experiments. The shaded region indicates the open-loop liftoff phase, and dashed segments after 1.5 s indicate the loss of controlled flight.
}
}
      \label{f:exp_response}
  \end{subfigure}
\centering
  \begin{subfigure}[b]{0.85\linewidth}
      \centering
      \includegraphics[width=8cm]{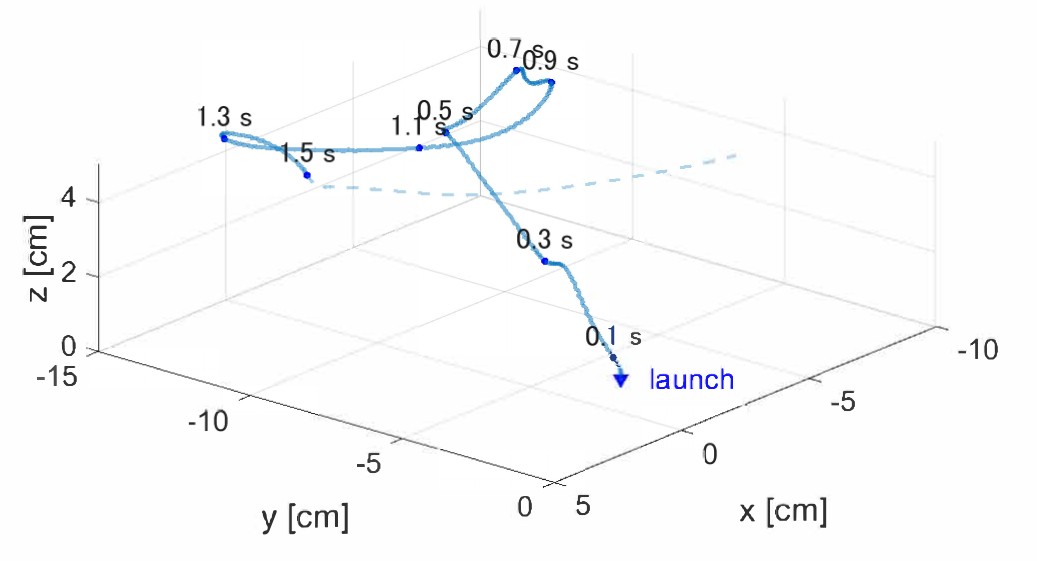}
      \caption{
      \textcolor{black}{
      Three-dimensional flight trajectory.The solid line shows the controlled flight period, while the dashed segment indicates the motion after the loss of controlled flight.
      }
      }
      \label{f:exp_traj}
  \end{subfigure}
\caption{
\textcolor{black}{
Experimental results
}
\label{f:exp}
}
\vspace{-3mm}
\end{figure}
\begin{figure*}[t!]
\begin{center}
\includegraphics[width=\linewidth]{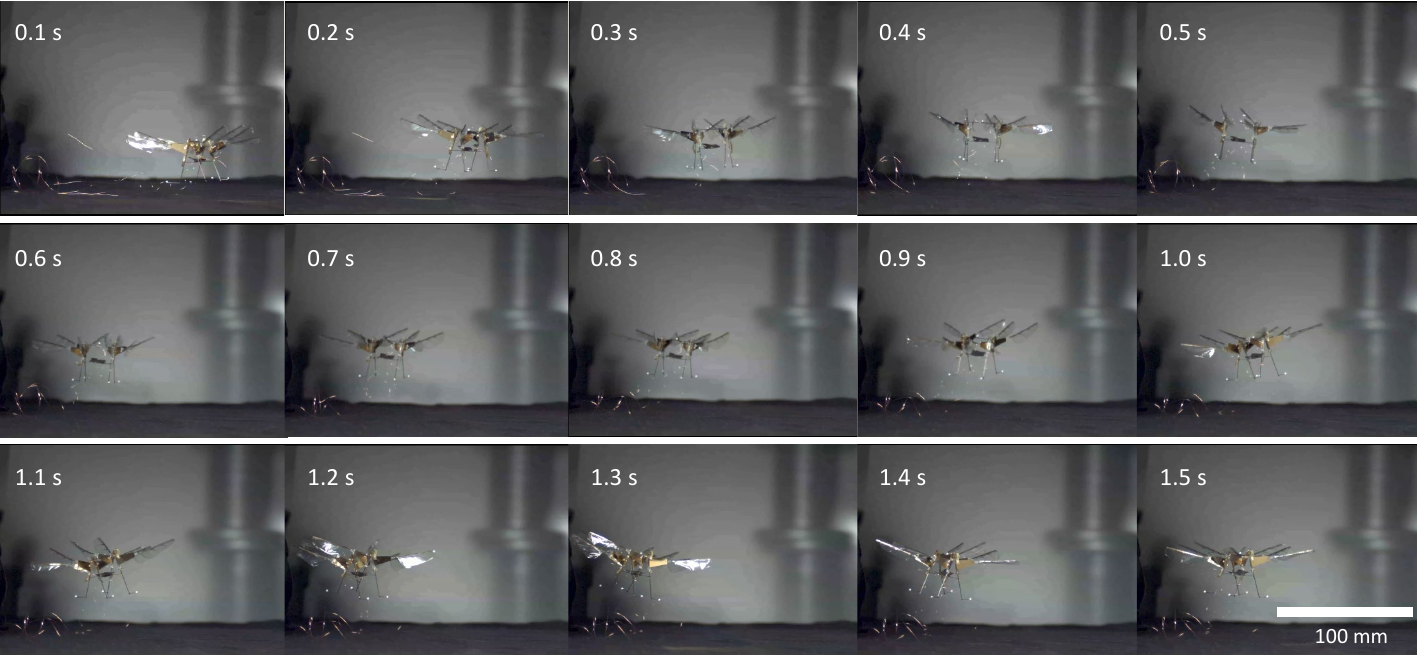} 
\caption{
\textcolor{black}{
Photographic sequence of a controlled flight.
}
\label{f:Frames_Ctrl}
}
\end{center}
\end{figure*}

\subsection{Results}

The results of a flight experiment with the four tilted paired-wing robot using the proposed controller are shown in Fig.~\ref{f:exp}. 
The upper graph in Fig.~\ref{f:exp_response} shows the yaw angle of the three paired-wing robot in \citep{2020_Jimbo}.
Fig.~\ref{f:Frames_Ctrl} shows the sequential shots taken during the controlled flight. 
Here, the target values are set to $\psi_d = 0$, $z_d = 0.05$, and $v_{x,d}^B=v_{y,d}^B=0$. As $v_{x,d}^B=v_{y,d}^B=0$, the target values in the absolute coordinate system are also $v_{x,d}=v_{y,d}=0$. Considering the weak yaw torque, the target angular velocity $\boldsymbol{\omega}_d$ of the attitude control is derived from the target roll angle $\phi_d$, target pitch angle $\theta_d$, and current yaw angle $\psi$.

\textcolor{black}{
The yaw angle $\psi$ shown in the upper graph of Fig.~\ref{f:exp_response} gradually deviates from the target value $\psi_d=0$ but converges to a final value.
}
The amount of the yaw drift is suppressed compared to the yaw angle of the three-paired-wing robot.
On the other hand, because controlling the yaw angle reduces the lift force in the vertical direction,
the altitude in the middle graph of Fig.~\ref{f:exp_response} oscillates, 
but it can be controlled to tend to the target value of $z_d=0.05$ m.
Furthermore, the translational velocity in the lower graph of Fig.~\ref{f:exp_response} is mostly controlled around the target value $v_{x,d}=v_{y,d}=0$. 
The oscillation is caused by ignoring the infinitesimal forces $f_{body,x}$ and $f_{body,y}$. 
Fig.~\ref{f:exp_traj} shows the flight trajectory. The flight time was approximately $1.5$ s.

\textcolor{black}{
Unlike the LQR-based flight experiments in \citep{npjRobotics2025}, which assumed an approximately trimmed and aligned condition before feedback control, the present experiment did not rely on such an assumption and instead compensated misalignment effects online during flight via adaptive control.
}

\textcolor{black}{
To quantitatively evaluate the experimental performance, the following metrics were computed. At the simulation settling time of 1.1 s ($\pm10\%$), the experimental altitude response remained within a $\pm 12\%$ band of the target altitude of 5 cm, corresponding to an experimental settling time of 1.08 s under this $\pm 12\%$ criterion. In terms of steady-state performance, the root mean square error (RMSE) of the altitude, evaluated over the interval from 1.0 s to 1.5 s, was 0.53 cm (approximately $10\%$ of the target altitude). The RMSEs of the horizontal velocities, $\dot{x}$ and $\dot{y}$, over the 1.5 s flight duration were 16.0 cm/s and 18.0 cm/s, respectively. In addition, the yaw drift was suppressed to within 100 deg, compared to approximately 850 deg for the three-paired-wing robot.
}

\section{Conclusion}

In this study, a flapping-wing robot with four paired tilted wings was developed to enable simultaneous control of four states, including yaw angle. 
\textcolor{black}{
The controllability Gramian was derived to quantify the controllability of the four states 
and to analyze the control characteristics of the fabricated configuration.
}
The wings are directly driven by piezoelectric actuators without transmission, and lift control is achieved simply by changing the voltage amplitude. 
However, misalignment or errors in the lift force model could cause an offset. 
Therefore, we designed an adaptive controller to alleviate the offset problem. Numerical experiments confirm that the proposed controller shows improved control performance compared to the LQI controller by adapting to unknown lift offsets. Finally, a tethered, controlled flight was performed, and the yaw drift was suppressed by the tilted-wing arrangement and the proposed controller.
Long-duration controlled flights will require on-board power. Controlled flights of battery-powered, tailless, flapping-wing robots weighing less than $10$ g are still challenging. 
In \citep{2023_5th_Ozaki}, we reported the first takeoff of the lightest battery-powered, tailless, flapping-wing robot ($2.1$ g insect scale), with a lift-to-power efficiency of approximately $5$ gf/W. 
The robot was equipped with a low-power circuit designed for digital duty-ratio control, a sensor unit, and a LiPo battery; however, it was not a controlled flight.  
In the future, we will incorporate the tilted-wing arrangement and the proposed controller into the system developed in \citep{2023_5th_Ozaki} to accomplish untethered battery-powered controlled flight.
Furthermore, to extract sufficient yaw torque, it is essential to improve the generation of lift force.
%

\appendix
\section{Proof of Theorem \ref{thm:attitude}}

Consider the following Lyapunov function candidate for the attitude control system:
\begin{equation}
V_\omega = s_\omega^\top (T J) s_\omega/2 + \tilde{\tau}_o^\top \Gamma_\omega^{-1} \tilde{\tau}_o/2.
\label{eq:Lyapunov_omega}
\end{equation}
where $\tilde{\tau}_o = \hat{\tau}_o - \tau_o$.
Note that the dynamics of $s_\omega$ is derived from (\ref{eq:Euler}), (\ref{eq:tau-delay}) and (\ref{eq:s_omega}) with (\ref{eq:adapt_tau}):  
\begin{equation}
T J \dot s_\omega + (J+K_\omega)s_\omega = \tilde{\tau}_o
\label{eq:s_omega_dynamics}
\end{equation}
The time derivative of (\ref{eq:Lyapunov_omega}) becomes a negative definite considering (\ref{eq:adapt_tauo}) and (\ref{eq:s_omega_dynamics}): 
\begin{equation*}
\dot{V}_\omega 
= s_\omega^\top (T J \dot{s}_\omega) + \tilde{\tau}_o^\top \Gamma_\omega^{-1} \dot{\hat{\tau}}_o 
= -s_\omega^\top (J + K_\omega) s_\omega \le 0.
\end{equation*}
As a result, 
the attitude control system is theoretically, globally, and asymptotically stable according to the invariant set theorem to satisfy $s_\omega \rightarrow 0$ as $t \rightarrow \infty$.
\section{Proof of Theorem \ref{thm:vertical}}
Similar to Appendix A, consider the following Lyapunov function candidate for the vertical control system:
\begin{equation}
V_z = T{s^2_z}/2 + {\tilde{f}^{~2}_{o,z}}/(2\gamma_z)
\label{eq:Lyapunov_z}
\end{equation}
where $\tilde{f}_{o,z} = \hat{f}_{o,z} - f_{o,z}$.
Note that the dynamics of $s_z$ is derived from (\ref{eq:vz}), (\ref{eq:f-delay}), and (\ref{eq:s_vz}) or (\ref{eq:s_z}) with (\ref{eq:adapt_f}):  
%
\begin{equation}
T \dot s_z + (1+k_z)s_z = \tilde{f}_{o,z}/m
\label{eq:s_z_dynamics}
\end{equation}
The time derivative of (\ref{eq:Lyapunov_z}) becomes a negative definite considering (\ref{eq:adapt_fo}) and (\ref{eq:s_z_dynamics}): 
\begin{equation*}
\dot{V}_z 
=  T s_z \dot{s}_z + \tilde{f}_{o,z}  \dot{\hat{f}}_{o,z}/\gamma_z
= -(1+k_z) s_z^2 \le 0.
\end{equation*}
As a result, 
the vertical control system is theoretically, globally, and asymptotically stable according to the invariant set theorem to satisfy $s_z \rightarrow 0$ as $t \rightarrow \infty$.
%




\bibliographystyle{elsarticle-num-names}


\bibliography{CEP2025mav4}

\end{document}